\documentclass[]{ifacconf}

\usepackage{enumitem}
\usepackage{graphicx}
\usepackage{natbib}
\let\theoremstyle\relax
\usepackage{amsmath,amsfonts,amssymb,bm,amsthm}
\usepackage{xcolor}
\usepackage{soul}
\usepackage{tikz}
\usetikzlibrary{angles, quotes}
\usepackage{url}
\usepackage{tikz}
\usetikzlibrary{positioning,arrows.meta}

\usepackage{comment}

\newtheorem{assumption}{Assumption}
\newtheorem{lemma}{Lemma}
\newtheorem{proposition}{Proposition}

\usepackage{todonotes}

\begin{document}

\begin{frontmatter}

\title{Wrench-Aware Admittance Control for Unknown-Payload Manipulation} 


\author[First]{Hossein Gholampour} 
\author[First]{Logan E. Beaver}

\address[First]{Department of Mechanical \& Aerospace Engineering, Old Dominion University, Norfolk, VA 23529, USA (e-mail: \{mghol004, lbeaver\}@odu.edu).}

\begin{abstract}                
Unknown payloads can strongly affect compliant robotic manipulation, especially when the payload center of mass is not aligned with the tool center point. In this case, the payload generates an offset wrench at the robot wrist. During motion, this wrench is not only related to payload weight, but also to payload inertia. If it is not modeled, the compliant controller can interpret it as an external interaction wrench, which causes unintended compliant motion, larger tracking error, and reduced transport accuracy. This paper presents a wrench-aware admittance control framework for unknown-payload pick-and-place using a UR5e robot. The method uses force-torque measurements in two different roles. First, a three-axis translational excitation term is used to reduce payload-induced force effects during transport without making the robot excessively stiff. Second, after grasping, the controller first estimates payload mass for transport compensation and then estimates the payload CoM offset relative to the TCP using wrist force-torque measurements collected during the subsequent translational motion. This helps improve object placement and stacking behavior. Experimental results show improved transport and placement performance compared with uncorrected placement while preserving compliant motion.
\end{abstract}


\end{frontmatter}

\section{INTRODUCTION}

Compliant control is important in robotic manipulation when the robot operates in uncertain environments or interacts physically with objects. Admittance control is widely used in such cases because it modifies robot motion according to measured external wrench, making it suitable for compliant manipulation on position-controlled industrial robots \citep{hogan1985impedance,Keemink2018,suomalainen2022survey}. This is useful in tasks such as material handling, packaging, bin-to-shelf transfer, and warehouse automation, where the robot must follow a motion plan while remaining responsive to contact even in the presence of modeling uncertainty.

A major difficulty arises when the robot manipulates an unknown payload. If the payload mass and mass distribution are not included in the nominal model, the wrench measured by the wrist-mounted force-torque (FT) sensor includes both the payload-induced wrench and the true interaction wrench, which cannot be separated directly. In compliant control, this uncompensated wrench influences the commanded motion, which can produce unintended compliant behavior and larger tracking error during transport. This issue becomes more pronounced when the payload center of mass (CoM) is not aligned with the tool center point (TCP). In that case, the payload generates an offset wrench at the wrist. Under motion, this is not limited to a static gravity moment; it also includes inertial force and moment components that can further disturb compliant motion if unmodeled or uncompensated \citep{GHOLAMPOUR2025,wen2025universal}. Accurate trajectory tracking under disturbance is also important in manipulation tasks, since deviations during transport can propagate to the final task outcome \citep{Gholampour2025Tracking}.

Increasing virtual stiffness is one way to limit payload-induced sag, but this weakens compliance and reduces the practical benefit of compliant manipulation. Recent work has studied adaptive and variable admittance strategies for uncertain interaction, including online parameter adaptation and improved disturbance handling \citep{Keemink2018,siciliano2016springer}. Our previous work addressed part of this problem by estimating unknown payload mass online and using an excitation force to reduce vertical sag during waypoint tracking \citep{GHOLAMPOUR2025}. That approach improved transport performance but was primarily concerned with payload-force effects and used objects with approximately symmetric mass distribution. In many practical tasks, the object CoM is offset from the TCP, so both transport and placement are affected by the resulting wrist wrench.

The CoM offset also matters in placement and stacking tasks. Even when the robot reaches the nominal placement pose, an off-centered object may not be placed in a balanced way after release. Recent work has shown that force-torque sensing can improve placement robustness under geometric or sensing uncertainty \citep{lerner2024precise,park2025experimental}, and that wrist wrench information can be used to estimate object mass and CoM for manipulation tasks \citep{Wang2021parameter,wen2025universal}. These results motivate using wrist wrench measurements not only for compliant transport, but also for placement correction.

This paper presents a wrench-aware admittance control framework for compliant pick-and-place of unknown payloads. Force-torque measurements serve two roles. First, during transport, a three-axis translational excitation term reduces payload-induced force effects within the admittance loop without increasing virtual stiffness. Second, after grasping, the controller first estimates payload mass for transport compensation and then estimates the payload CoM offset relative to the TCP using wrist force-torque measurements collected during the subsequent translational motion. The filtered offset estimate is then used to correct the horizontal placement waypoint so that the object, rather than just the TCP, is placed closer to its equilibrium location on the support surface.Unlike approaches that estimate payload properties in a separate static or dedicated identification stage, the present method performs CoM-offset estimation during the manipulation task itself.

The method is implemented on a UR5e manipulator with Cartesian admittance control and waypoint-based motion. The main contribution is a task-embedded wrench-aware manipulation framework that estimates the payload CoM offset online during a translational admittance-motion segment and uses that estimate directly for object-level placement correction, while also compensating payload-induced transport forces within the same compliant control loop. Repeated placement and stacking trials serve as additional validation of the estimated payload offset.

The rest of the paper is organized as follows. Section~\ref{sec:Formulation} presents the problem formulation. Section~\ref{sec:payload_model} presents the mass-estimation and transport-compensation model and develops the payload offset-estimation framework used for placement correction. Section~\ref{sec:results} presents the experimental setup and results. Section~\ref{sec:conclusion} concludes the paper.

\section{Problem Formulation} \label{sec:Formulation}

We consider a Cartesian admittance control framework for a robotic manipulator equipped with a wrist-mounted FT sensor. Let $\bm{p}_a \in \mathbb{R}^3$ denote the position generated by the admittance model, and let $\bm{p}_{ref} \in \mathbb{R}^3$ denote the reference position trajectory. We write the standard translational admittance model as
\begin{equation}
\bm{F} = M_a (\ddot{\bm{p}}_a - \ddot{\bm{p}}_{ref}) + B_a (\dot{\bm{p}}_a - \dot{\bm{p}}_{ref}) + K_a (\bm{p}_a - \bm{p}_{ref}) - \bm{F}_{\mathrm{exc}},
\label{eq:adm3d}
\end{equation}
where $M_a, B_a, K_a \in \mathbb{R}^{3 \times 3}$ are the virtual inertia, damping, and stiffness matrices, respectively, and $\bm{F}_{\mathrm{exc}} \in \mathbb{R}^3$ is an excitation or compensation term. 

For waypoint-level motion, the reference is piecewise constant within each segment, so $\dot{\bm{p}}_{ref}=0$ and $\ddot{\bm{p}}_{ref}=0$. Solving \eqref{eq:adm3d} for the admittance acceleration gives
\begin{equation}
\ddot{\bm{p}}_a
=
M_a^{-1}\!\left(
\bm{F} - B_a \dot{\bm{p}}_a - K_a(\bm{p}_a - \bm{p}_{ref}) + \bm{F}_{\mathrm{exc}}
\right),
\label{eq:adm_acc}
\end{equation}
and we obtain the commanded Cartesian velocity by integrating,
\begin{equation}
\dot{\bm{p}}_a(t) = \dot{\bm{p}}_a(t_0) + \int_{t_0}^{t} \ddot{\bm{p}}_a(\tau)\, d\tau .
\label{eq:adm_vel}
\end{equation}

Although the admittance model in \eqref{eq:adm3d} is translational, unknown payloads can still affect transport performance because the measured wrist force includes payload-induced force components in addition to the true interaction force. When the payload CoM is offset from the TCP, the measured wrist moment also contains information about that offset, which we later use in the payload-wrench model for placement correction.

\begin{assumption}
\label{ass:pure_translation}
During the transport and CoM-estimation segments, the payload remains rigidly attached to the TCP and the end effector undergoes translational motion with negligible angular velocity and angular acceleration, so that
\(\bm{\omega} \approx \bm{0},
\qquad
\bm{\alpha} \approx \bm{0}.\)
\end{assumption}

\begin{assumption}
\label{ass:wrench_quality}
During the CoM-estimation segment, the measured wrist wrench is sufficiently accurate, and the translational samples collected over that segment are sufficiently diverse that the stacked offset-estimation system has rank $3$.
\end{assumption}

\section{Mass Estimation and Transport Compensation} \label{sec:payload_model}

\subsection{Mass Estimation and Transport Compensation} \label{sec:transport_comp}
Building on the dynamic payload-mass estimate in \citet{GHOLAMPOUR2025}, we model the measured wrist force during translational motion as the sum of true interaction force and payload-induced force. We write
\begin{equation}
\bm{F} = \bm{F}_{\mathrm{int}} + \bm{F}_p,
\label{eq:force_decomp}
\end{equation}
where $\bm{F}_{\mathrm{int}} \in \mathbb{R}^3$ is the true interaction force and $\bm{F}_p \in \mathbb{R}^3$ is the payload-induced translational force. Substituting \eqref{eq:force_decomp} into \eqref{eq:adm_acc} gives
\begin{equation}
    \ddot{\bm{p}}_a = M_a^{-1}\!\left(\bm{F}_{\mathrm{int}} + \bm{F}_p - B_a \dot{\bm{p}}_a - K_a(\bm{p}_a - \bm{p}_{ref}) + \bm{F}_{\mathrm{exc}}
    \right).
    \label{eq:adm_force_split}
\end{equation}

The purpose of transport compensation is to reduce the effect of $\bm{F}_p$ without increasing virtual stiffness. Defining the residual payload force
\begin{equation}
    \tilde{\bm{F}}_p := \bm{F}_p - \bm{F}_{\mathrm{exc}},
    \label{eq:force_residual}
\end{equation}
gives
\begin{equation}
    \ddot{\bm{p}}_a = M_a^{-1}\!\left( \bm{F}_{\mathrm{int}} + \tilde{\bm{F}}_p - B_a \dot{\bm{p}}_a - K_a(\bm{p}_a - \bm{p}_{ref}) \right).
    \label{eq:adm_residual}
\end{equation}
Thus, if $\bm{F}_{\mathrm{exc}}$ approximates $\bm{F}_p$, the admittance response is driven mainly by the true interaction force.

Following \citet{GHOLAMPOUR2025}, we estimate the payload mass during translational motion as
\begin{equation}
\hat{m}_u(k) = \frac{f_{p,z}(k)}{a_{TCP,z}(k) - g},
\label{eq:mass_est}
\end{equation}
where $f_{p,z}(k)$ denotes the payload-induced vertical force component and $a_{TCP,z}(k)$ is the measured TCP linear acceleration in the same direction. We then form the translational payload-force estimate as
\begin{equation}
\bm{F}_{\mathrm{exc}} = \hat{\bm{F}}_p,
\qquad
\hat{\bm{F}}_p(k) = \hat{m}_u(k)\big(\bm{a}_{TCP}(k) - \bm{g}\big).
\label{eq:fp_3d_est}
\end{equation}
Equation \eqref{eq:fp_3d_est} is the payload-force model used in the transport-compensation law.

\subsection{Payload Offset Estimation from Wrist Wrench} \label{sec:offset_est}

Under Assumption~\ref{ass:pure_translation}, Let $\bm{f} \in \mathbb{R}^3$ and $\bm{\tau} \in \mathbb{R}^3$ denote the measured force and moment at the wrist. When the payload CoM is not aligned with the TCP, the measured moment is influenced both by the CoM offset and by the payload rotational dynamics. The CoM linear acceleration satisfies
\begin{equation}
\bm{a}_C = \bm{a}_{TCP} + \bm{\alpha} \times \bm{r} + \bm{\omega} \times (\bm{\omega} \times \bm{r}),
\label{eq:aC_relation}
\end{equation}
where $\bm{a}_{TCP}$ is the TCP linear acceleration, $\bm{r} \in \mathbb{R}^3$ is the CoM offset relative to the TCP, and $\bm{\omega},\bm{\alpha} \in \mathbb{R}^3$ are the angular velocity and angular acceleration. The rigid-body wrench relation can then be written as
\begin{equation}
\bm{f} = m(\bm{a}_C - \bm{g}),
\label{eq:force_dynamic}
\end{equation}
\begin{equation}
\bm{\tau} = \bm{r} \times \bm{f} + I_C \bm{\alpha} + \bm{\omega} \times I_C \bm{\omega},
\label{eq:moment_full}
\end{equation}
where $m$ is the payload mass, $g$ is gravity, and $I_C$ is the payload inertia tensor about the CoM.

Equation \eqref{eq:force_dynamic} shows that the CoM offset affects the translational payload force during motion, while \eqref{eq:moment_full} shows its effect on the wrist moment. Under Assumption~\ref{ass:pure_translation}, the rotational inertia term
\begin{equation}
\bm{\tau}_{\mathrm{rot}} := I_C \bm{\alpha} + \bm{\omega} \times I_C \bm{\omega}
\label{eq:tau_rot}
\end{equation}
is negligible, so \eqref{eq:moment_full} reduces to
\begin{equation}
\bm{\tau} \approx \bm{r} \times \bm{f}.
\label{eq:moment_reduced}
\end{equation}

\begin{lemma} \label{lem:rvector}
For a nonzero force sample $\bm{f} \in \mathbb{R}^3$, \eqref{eq:moment_reduced} does not uniquely determine $\bm{r}$.
\end{lemma}

\begin{proof}
Assume that $\bm{r}$ is one solution of \eqref{eq:moment_reduced}. For any scalar $\lambda \in \mathbb{R}$, define
\begin{equation*}
\tilde{\bm{r}} = \bm{r} + \lambda \bm{f}.
\end{equation*}
Then
\begin{equation*}
\tilde{\bm{r}} \times \bm{f}
=
(\bm{r} + \lambda \bm{f}) \times \bm{f}
=
\bm{r} \times \bm{f} + \lambda (\bm{f} \times \bm{f})
=
\bm{r} \times \bm{f}
=
\bm{\tau}.
\end{equation*}
This completes the proof.
\end{proof}
Any component of $\bm{r}$ parallel to $\bm{f}$ is therefore unobservable from the single wrench sample, so the full three-dimensional offset is not uniquely determined.

\begin{proposition} \label{prop:dynamic_offset}
Under Assumption~\ref{ass:pure_translation}, consider measured wrench samples
\[
\{(\bm{f}_k,\bm{\tau}_k)\}_{k=1}^N
\]
collected over a translational segment. Then
\begin{equation}
\bm{\tau}_k \approx \bm{r} \times \bm{f}_k
= -[\bm{f}_k]_\times \bm{r},
\qquad k=1,\dots,N,
\label{eq:stacked_moment_samples}
\end{equation}
defines an approximate linear system in the unknown offset vector $\bm{r}$, where $[\bm{f}_k]_\times$ denotes the skew-symmetric matrix associated with $\bm{f}_k$. The offset is identifiable from the stacked samples if the coefficient matrix
\begin{equation}
A :=
\begin{bmatrix}
-[\bm{f}_1]_\times \\
\vdots \\
-[\bm{f}_N]_\times
\end{bmatrix}
\label{eq:A_stacked}
\end{equation}
has rank $3$.
\end{proposition}

\begin{proof}
Under Assumption~\ref{ass:pure_translation}, \eqref{eq:moment_full} reduces to \eqref{eq:moment_reduced}. Applying \eqref{eq:moment_reduced} to each sample gives \eqref{eq:stacked_moment_samples}. Stacking these equations yields a linear system in $\bm{r}$. The solution is unique when the stacked coefficient matrix has rank $3$. This completes the proof.
\end{proof}
Lemma \ref{lem:rvector} shows that a single wrench sample does not uniquely determine $\bm{r}$. Proposition \ref{prop:dynamic_offset} shows that samples collected during translational motion can be combined to estimate the payload offset vector from measured force-torque data.

\section{Experimental Setup and Results} \label{sec:results}

\subsection{Experimental Setup}
Figure \ref{fig:task_schematic} shows the designed pick-and-place sequential task in the X-Z plane. The robot first approaches the object through Waypoints 1 and 2, grasps it at Waypoint 3, and then moves toward Waypoint 4. During this post-grasp translational motion, the controller first estimates payload mass for transport compensation and then estimates the payload CoM offset relative to the TCP. The robot then moves through Waypoint 5 and places the object at Waypoint 6 on the equilibrium bar. The schematic also shows that the object CoM is offset from the geometric center along the object length. Therefore, the nominal placement point, which would align the object length center with the support location, is not sufficient for balanced placement. Instead, the controller corrects the release point using the measured offset so that the final placement aligns the support location with the object CoM.

\begin{figure}[t]
\centering
\includegraphics[width=\linewidth]{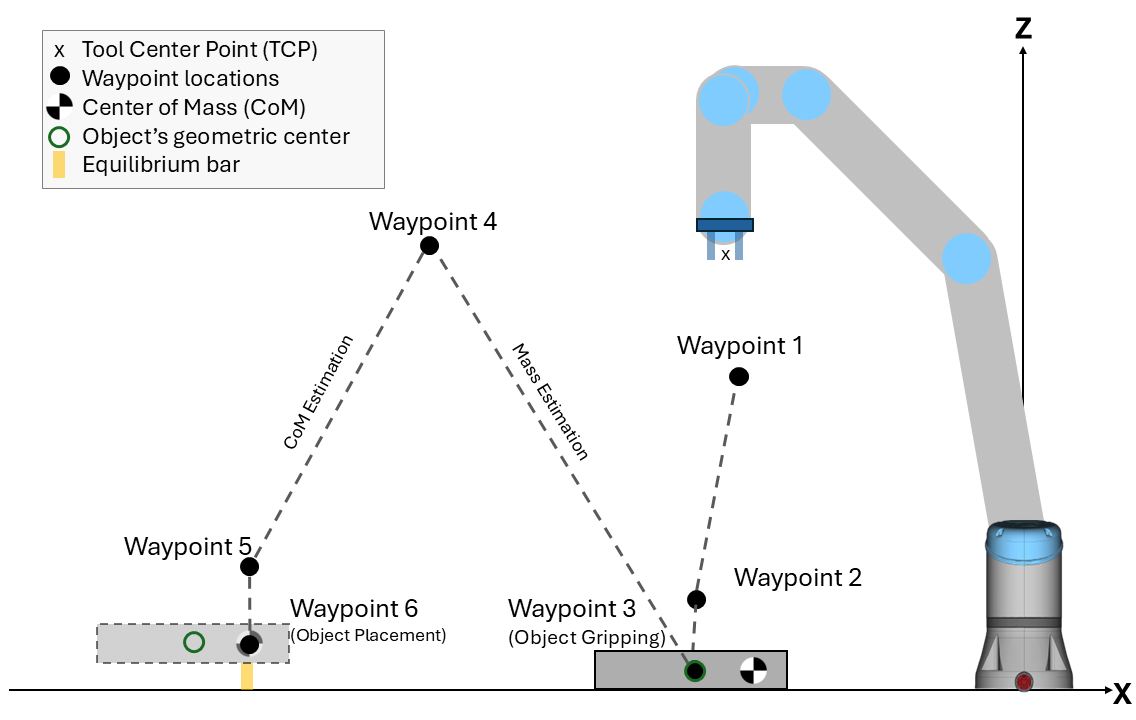}
\caption{Task schematic and planned waypoints in the X-Z plane.}
\label{fig:task_schematic}
\end{figure}

We implement the method on a UR5e manipulator equipped with a wrist-mounted FT sensor and a Robotiq 2F-140 gripper. A Linux PC communicates with the robot through RTDE and sends the Cartesian velocity commands generated by the translational admittance controller.
Our code is available online\footnote{\url{https://github.com/Hossein-ghd/ur5e-com-aware-admittance}}.

The validation task is a waypoint-based pick-and-place experiment using off-centered objects and a narrow support bar that defines the desired equilibrium line. We design the payload offset along the x direction, so the controller applies the placement correction in x while the task geometry sets the vertical placement coordinate. After grasping, the robot executes a translational measurement segment before placement. During this segment, the controller collects force-torque samples at the wrist and computes the payload offset estimate from the reduced wrench model developed in Section~\ref{sec:offset_est}. The filtered estimate then shifts the commanded TCP release position according to \eqref{eq:place_cmd}. In this way, the corrected command targets the object equilibrium location on the support, rather than only the nominal TCP placement point.

We designed the experiments to validate the estimated offset through the final object placement outcome, rather than only through TCP positioning accuracy. If the estimate is accurate, the corrected TCP release position places the object so that its CoM lies closer to the equilibrium line after release. This is important for off-centered objects because the TCP can reach a nominal placement point while the object itself still lands in an unbalanced position. We therefore use trials with different grasp locations on the same object, as well as repeated placement and stacking trials, to evaluate the consistency of the estimated offset and the corrected placement command. The green tape markers show the physical CoM location only as an offline reference for validation and plotting; the controller does not use them.

A practical issue in these experiments is grasp quality. The theoretical model assumes that the object remains rigidly attached to the gripper during the measurement and transport stages. In practice, however, off-centered grasps can generate a large moment at the fingertip-object interface and cause torque-induced micro-slip even when the payload weight is supported. This effect becomes more pronounced for larger CoM offsets and during robot motion. To reduce it, we added friction-enhancing material at the fingertip-object contact surfaces. This improves grasp repeatability and reduces wrench corruption caused by relative motion between the object and the gripper.

Figure~\ref{fig:exp_setup} shows the experimental setup used for the placement and stacking validation tasks.

\begin{figure}[ht]
\centering
\includegraphics[width=1\linewidth]{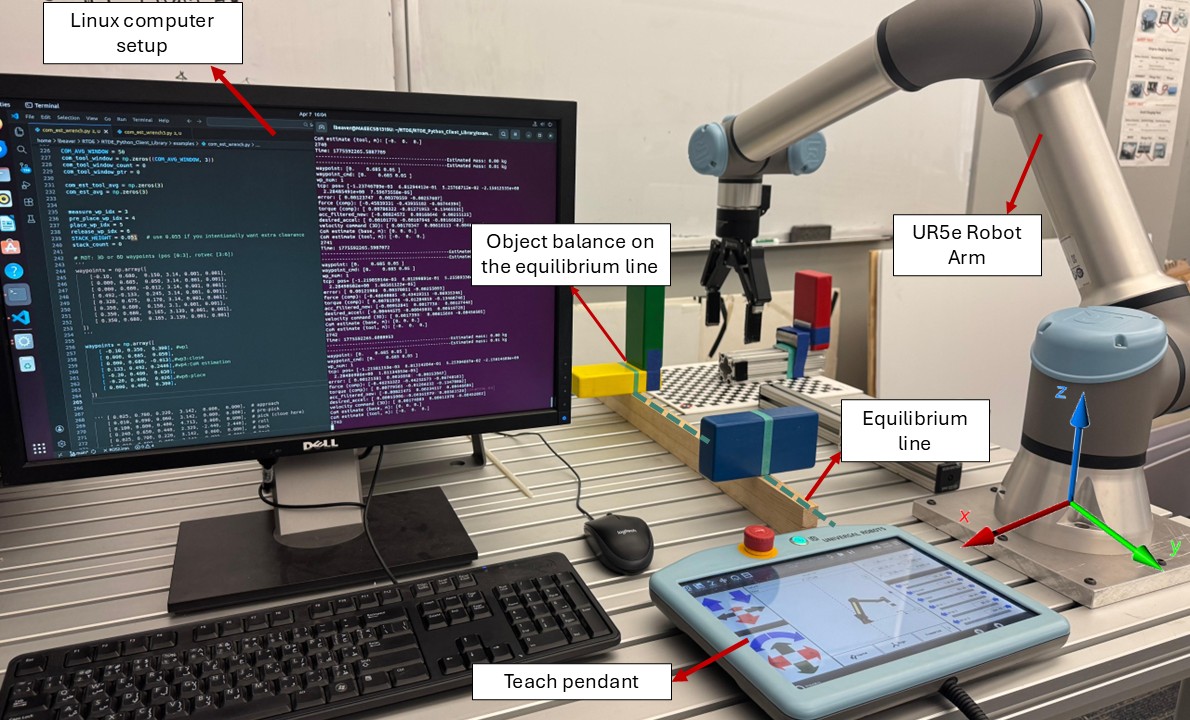}
\caption{Experimental setup used for placement and stacking validation.}
\label{fig:exp_setup}
\end{figure}

\subsection{Offset Estimation and TCP Response}

We next examine a representative trial during the short translational segment used for payload-offset estimation. Figure~\ref{fig:xcorr_time} shows the x-direction time histories for this stage.

The top plot shows the estimated CoM offset in the x direction, together with its filtered value and the reference offset used for validation over the gripping interval. A noticeable initial spike appears at the gripping moment. This transient corresponds to the pickup event, before the object settles in the gripper, and can reflect short grasp-induced motion, local contact effects, and payload sag before the transport-compensation estimate becomes effective. After this initial transient, and before the CoM measurement window begins, the controller uses this interval to estimate the payload mass and apply force compensation. Once the CoM measurement stage starts, the estimated x-offset converges toward the reference value and remains close to it over the rest of the plotted interval. For this representative trial, the x-offset RMSE, computed from the start of the measurement window to the end of the plotted interval, is 3.5 mm.

The bottom plot shows the TCP motion in the x direction over the same gripping interval, together with the ideal admittance trajectory. The measured TCP motion remains close to the ideal motion during the estimation stage. For this representative trial, the corresponding TCP x-position RMSE over the same interval is 1.2 mm. These results show that the controller can estimate the CoM offset during translational motion, without requiring a separate rotational or static measurement step, while keeping the TCP motion close to the desired response.

\begin{figure}[ht]
\centering
\includegraphics[width=0.92\linewidth]{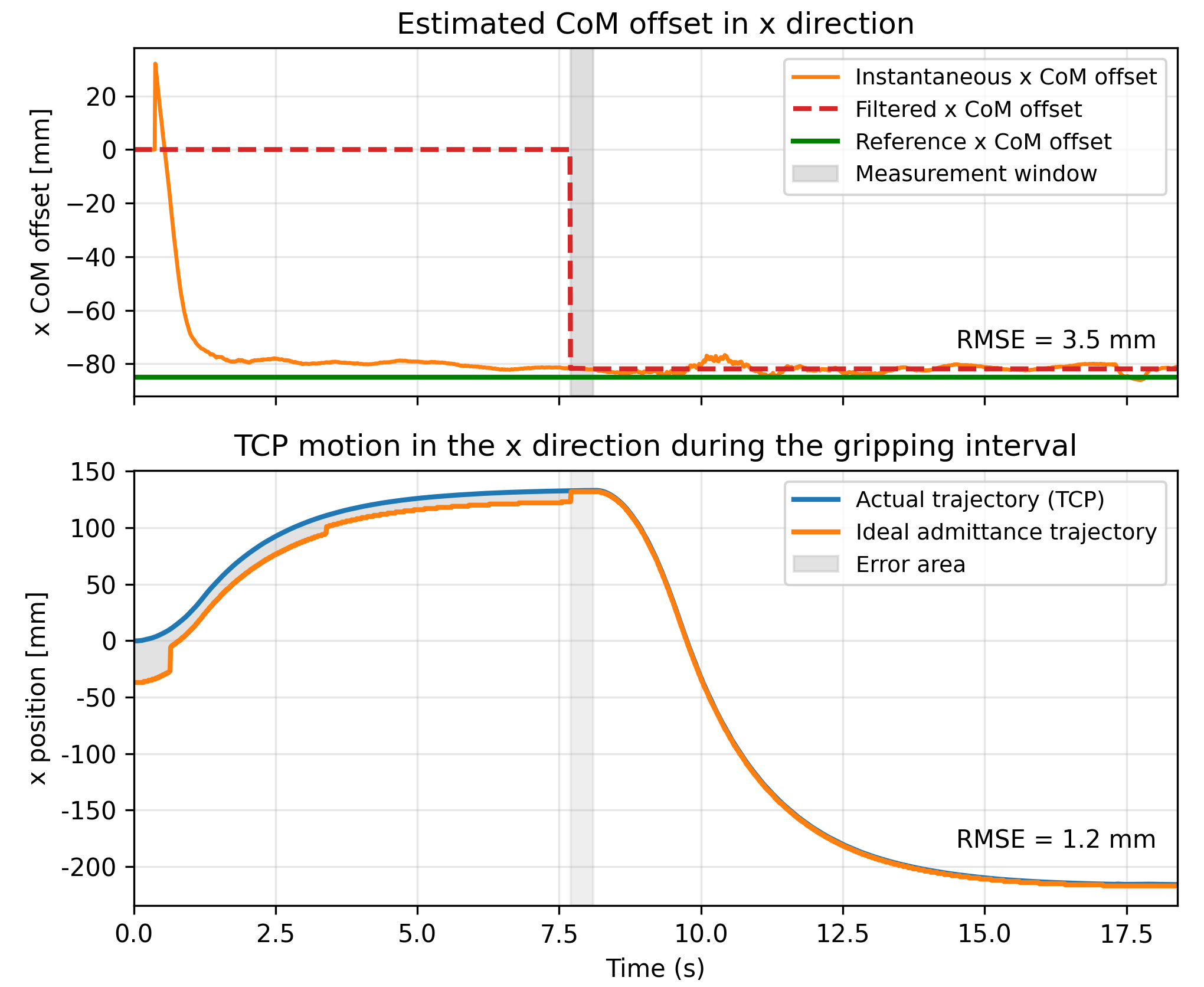}
\caption{Representative gripping-interval response in the x direction. Top: estimated CoM offset. Bottom: TCP motion.}
\label{fig:xcorr_time}
\end{figure}

\subsection{Placement Validation}
The CoM offset also affects the final placement stage. Let $\bm{p}_{\mathrm{place,nom}} \in \mathbb{R}^3$ denote the nominal TCP placement position. If the controller moves the TCP to $\bm{p}_{\mathrm{place,nom}}$ without accounting for the payload offset, the object itself may not reach the desired support location. This can lead to placement error, loss of equilibrium, or less reliable repeated placements, even when the transport trajectory is tracked reasonably well.

To account for this effect, the controller corrects the placement command using the estimated horizontal offset. Let $\hat{\bm{r}} = [\hat{r}_x,\hat{r}_y,\hat{r}_z]^T$ denote the estimated payload offset relative to the TCP. We define the corrected placement command as
\begin{equation}
\bm{p}_{\mathrm{place,cmd}} = \bm{p}_{\mathrm{place,nom}} - [\hat{r}_x,\hat{r}_y,0]^T.
\label{eq:place_cmd}
\end{equation}
In this study, we apply the correction only in the horizontal plane, while the task geometry fixes the vertical placement coordinate.

We next evaluate the final placement result for a representative trial. Figure~\ref{fig:placement_geom} shows the geometry in the support plane. The blue cross marks the desired equilibrium point on the support. The controller does not move the TCP to that point directly. Instead, it shifts the TCP release position by the estimated correction vector so that, after release, the object lies closer to its equilibrium location on the support.

The same figure compares the desired object pose with the final implemented pose. The green circle and dashed rectangle mark the desired corrected release position and the corresponding desired object pose. The red circle and solid rectangle mark the final release position and object pose after the controller applies the estimated correction and executes the placement command. The two poses remain close, which shows that the corrected release point is consistent with the desired object placement on the support. 
The remaining gap reflects both the correction-command error from the estimated CoM offset and a smaller execution error during the final placement motion.

Figure~\ref{fig:final_error_zoom} shows the estimated horizontal CoM offset in the TCP frame for the same representative trial. This figure is plotted in the TCP-relative offset coordinates $(r_x,r_y)$ and is not a zoom of the placement geometry in Fig.~\ref{fig:placement_geom}. The orange circles show the sample-by-sample offset estimates over the measurement window, the red circle shows the filtered offset estimate, and the green circle shows the reference offset used for validation. The remaining difference between the filtered estimate and the reference value represents the offset-estimation error.
Because the placement command uses the negative of the estimated offset in \eqref{eq:place_cmd}, the relative shift between the filtered and reference points in Fig.~\ref{fig:final_error_zoom} appears with the opposite sign in the placement positions shown in Fig.~\ref{fig:placement_geom}.

Table~\ref{tab:placement_summary} reports the representative placement values for this trial in the correction convention used for the TCP placement shift. The table lists the target equilibrium point, the ideal correction vector from the manual reference offset, the corresponding ideal corrected TCP point, the estimated correction vector obtained from the measurement stage, and the commanded TCP release position computed from that estimate. The difference between the ideal and estimated quantities is consistent with the offset-estimation error shown in Fig.~\ref{fig:final_error_zoom}.

\begin{figure}[t]
\centering
\includegraphics[width=\linewidth]{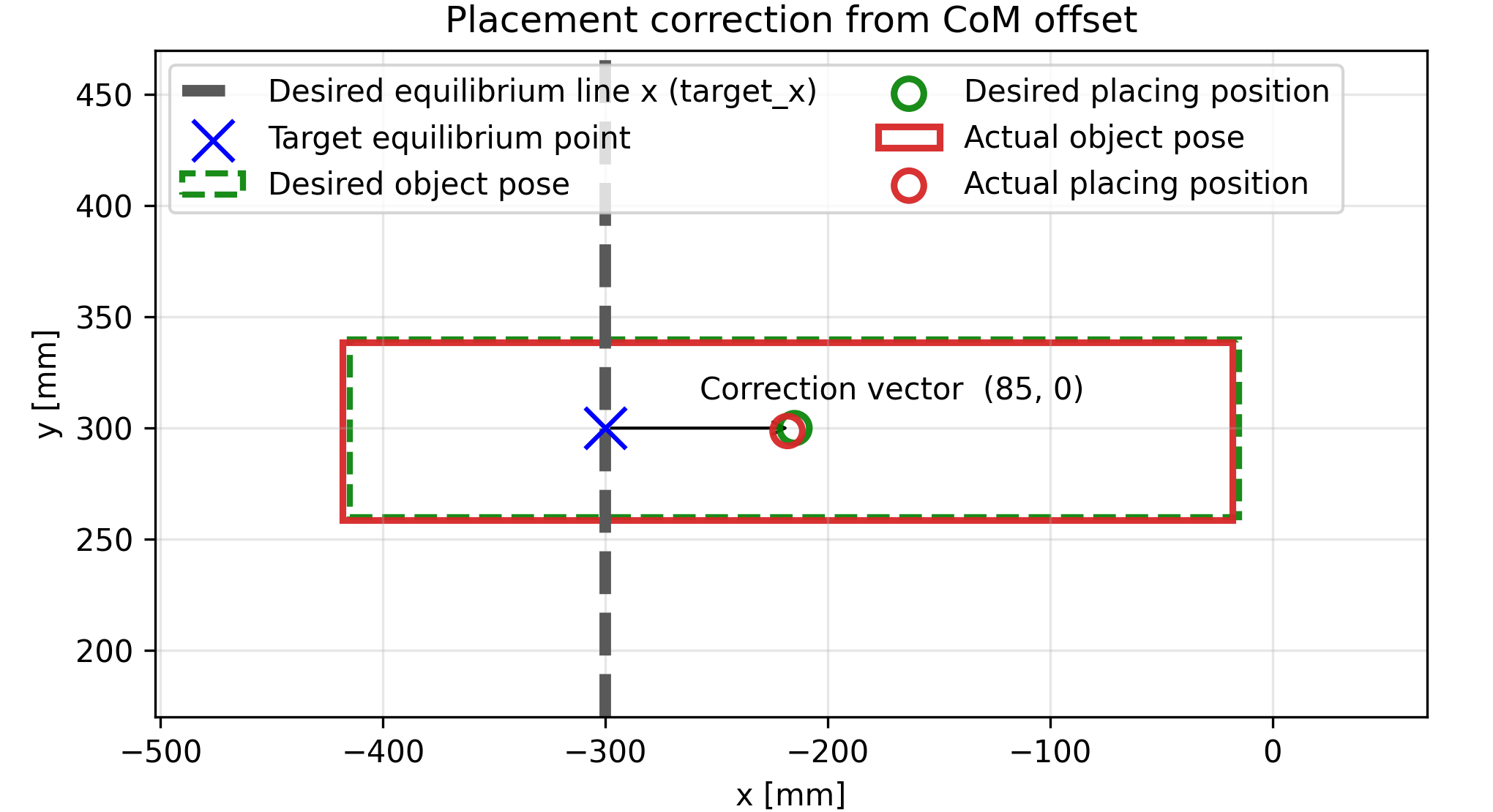}
\caption{Placement geometry for a representative trial.}
\label{fig:placement_geom}
\end{figure}

\begin{figure}[t]
\centering
\includegraphics[width=0.9\linewidth]{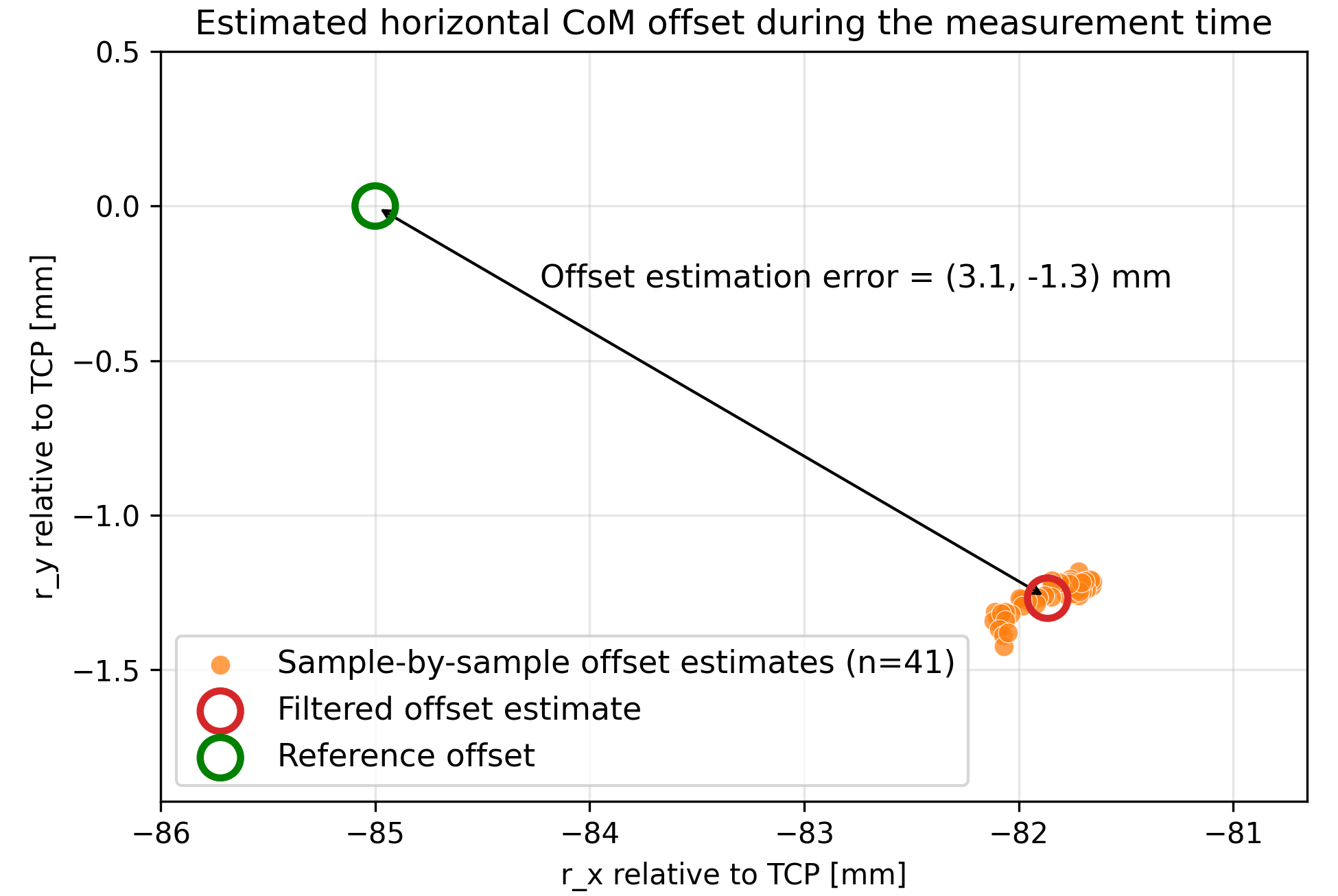}
\caption{Estimated horizontal CoM offset in the TCP frame for the representative trial.}
\label{fig:final_error_zoom}
\end{figure}

\begin{table}[t]
\caption{Representative placement geometry for one trial.}
\label{tab:placement_summary}
\centering
\footnotesize
\begin{tabular}{lc}
\hline
Quantity & Value [mm] \\
\hline
Target equilibrium point $(x,y)$ & $(-300.00,\ 300.00)$ \\
Ideal correction vector $(c_x,c_y)$ & $(85.00,\ 0.00)$ \\
Ideal corrected TCP $(x,y)$ & $(-215.00,\ 300.00)$ \\
Estimated correction vector $(c_x,c_y)$ & $(82.84,\ -1.66)$ \\
Commanded TCP release position $(x,y)$ & $(-217.16,\ 298.34)$ \\
Correction-command error & $2.73$ \\
Actual TCP release position $(x,y)$ & $(-218.12,\ 298.69)$ \\
Actual release error relative to ideal & $3.38$ \\
Execution error relative to command & $1.03$ \\
\hline
\end{tabular}
\end{table}

To show the effect of the correction more directly, we also compare the corrected case with a nominal placement case in which the controller does not apply the estimated CoM correction. Figure \ref{fig:uncorrected_case} shows this nominal case. The black cross marks the target equilibrium point, the red circle marks the final TCP position, and the red rectangle shows the final object pose after placement. The blue cross marks the actual CoM location. In this trial, the robot places the object body near the nominal target, but the object CoM remains displaced from the equilibrium line by the payload offset. The arrow from the TCP position to the CoM location shows this offset, and the separation between the CoM location and the equilibrium line shows the resulting alignment error. This result shows that the nominal placement can keep the object body near the commanded position while still failing to satisfy the object-level equilibrium condition.

\begin{figure}[ht]
\centering
\includegraphics[width=\linewidth]{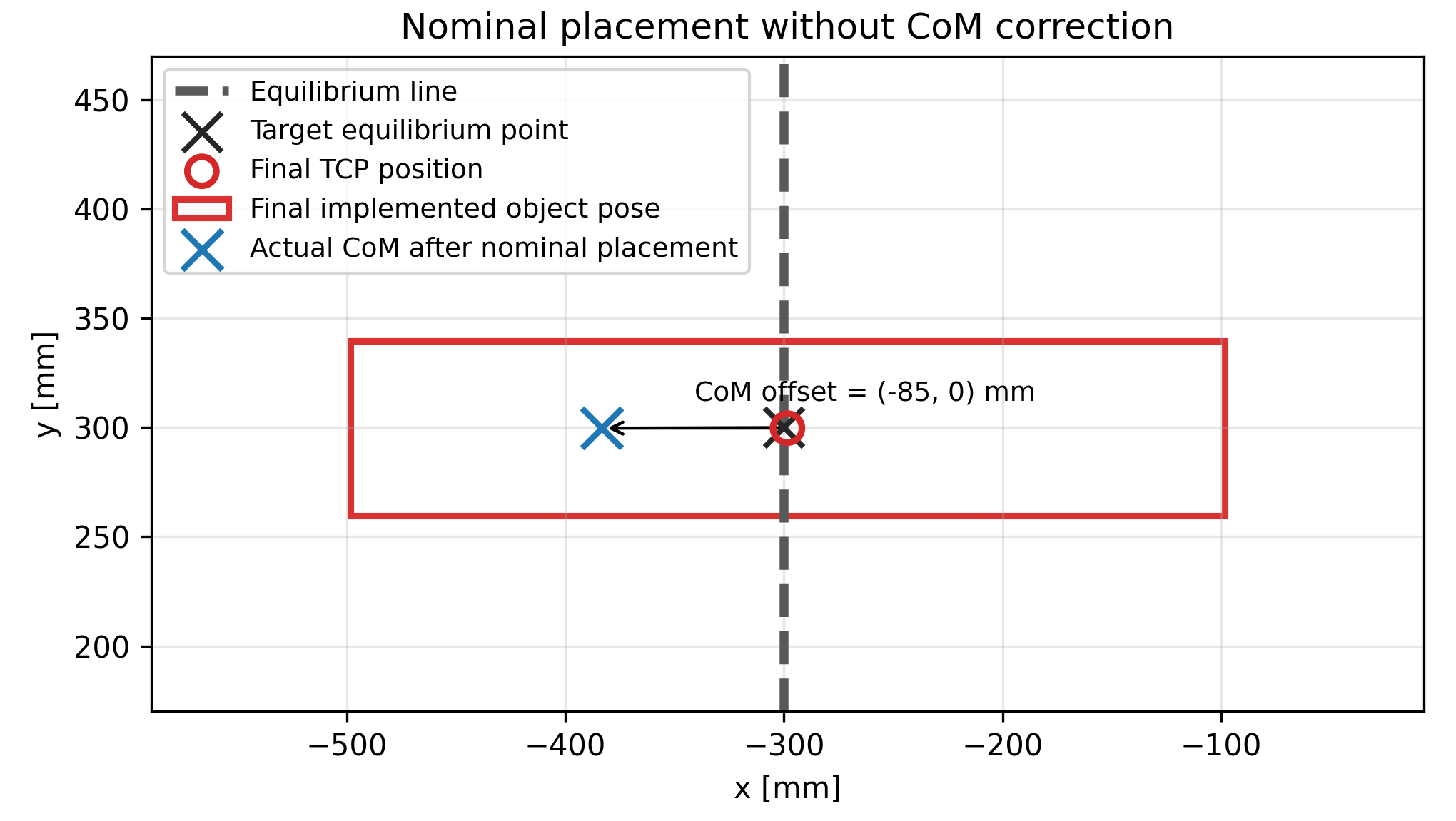}
\caption{Nominal placement without CoM correction.}
\label{fig:uncorrected_case}
\end{figure}

\subsection{Repeated Placement and Stacking Validation}
For repeated validation trials, we extend the corrected placement command across successive layers as
\begin{equation}
\bm{p}_{\mathrm{place,cmd}}^{(n)} = \bm{p}_{\mathrm{place,nom}} - [\hat{r}_x,\hat{r}_y,0]^T + [0,0,nh]^T,
\label{eq:place_cmd_stack}
\end{equation}
where $n$ is the layer index and $h$ is the object height or effective layer spacing.

We use these repeated placements only to validate the consistency of the offset estimate and the corrected placement command.

We finally evaluate repeated placement and stacking using the corrected placement rule in \eqref{eq:place_cmd_stack}. We designed this task to validate the estimated correction across successive layers rather than only for a single placement. Before the trials, we manually measured the physical CoM location of each object and marked it with the same green ribbon as an offline reference. During each task cycle, the controller estimated the CoM offset from wrist force-torque data during the translational measurement segment and used that estimate to compute the corrected release position.

Figure~\ref{fig:stacking_validation} shows a representative three-layer stacking result in the x-z plane. The dashed vertical line denotes the equilibrium line, and the green markers denote the projected CoM locations of the placed blocks. For all three layers, the projected CoM locations remain close to the equilibrium line after release, and the stacked blocks remain balanced.

Table~\ref{tab:stacking_summary} summarizes the stacking-validation results. The table lists the manual CoM x-location as the reference value, the estimated CoM x-location from the measurement stage, and the implemented CoM x-location after placement. The agreement among these values shows that the estimated correction remains effective across repeated placements and stacking steps.

\begin{figure}[ht]
\centering
\includegraphics[width=\linewidth]{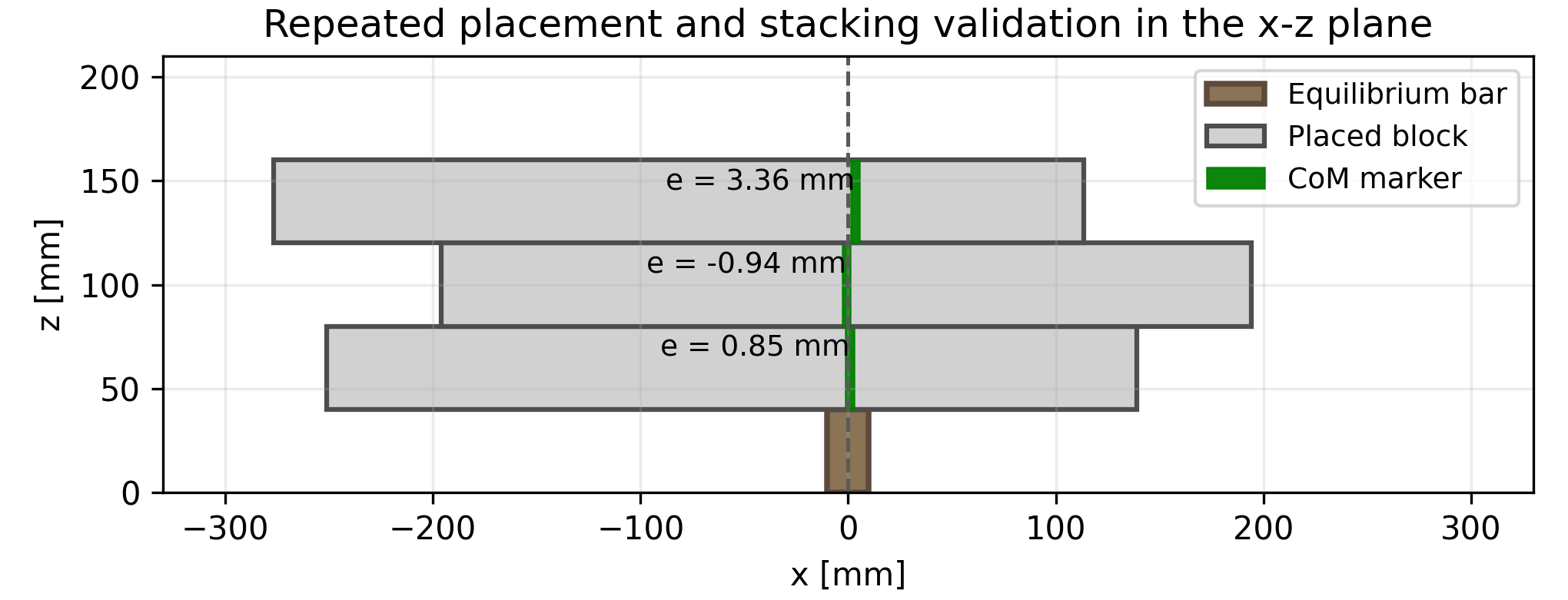}
\caption{Repeated placement and stacking validation in the x-z plane.}
\label{fig:stacking_validation}
\end{figure}

\begin{table}[ht]
\caption{Stacking validation summary across repeated placements.}
\label{tab:stacking_summary}
\centering
\footnotesize
\begin{tabular}{lccc}
\hline
Object & Actual CoM & Estimated CoM & Implemented CoM \\
\hline
Object 1 & -57.00 & -56.44 & -56.20 \\
Object 2 &  0.00& -0.57 & -0.61 \\
Object 3 &  -85.00& -81.91 & -81.60 \\
\hline
\end{tabular}
\end{table}

\section{Conclusions} \label{sec:conclusion}

This paper presented a wrench-aware admittance control framework for unknown-payload pick-and-place using a UR5e manipulator. The method used force-torque measurements in two roles within the same framework. First, a three-axis translational excitation term reduced payload-induced force effects during transport without increasing stiffness. Second, a short translational measurement stage after grasping provided a payload-offset estimate from wrench samples, and the controller used the filtered estimate to correct the horizontal placement command.

The experiments showed that the controller could estimate the payload offset during translational motion while keeping the TCP response close to the desired admittance behavior. In a representative trial, the x-offset RMSE over the post-measurement interval was 3.5 mm, the corresponding TCP x-position RMSE was 1.2 mm, and the corrected placement achieved a final TCP release error of 3.38 mm. The nominal trial further showed that accurate TCP placement alone does not guarantee balanced object placement when the payload CoM is offset from the TCP. Repeated stacking trials also supported the consistency of the estimated offset and the corrected placement command across successive layers.

Overall, the results show that force-based transport compensation and wrench-based placement correction can be integrated in a single admittance-control framework for unknown-payload manipulation. A natural next step is to relax the pure-translation assumption and extend the framework to rotational admittance, so that the controller can account for both translational and rotational payload effects during more general manipulation tasks.

\end{document}